\documentclass[a4paper,11pt]{article}
\usepackage[utf8]{inputenc}
\usepackage{amsmath,amsfonts,amssymb,amsthm}
\usepackage{fullpage}   % narrow page margins and set a fixed page style
\usepackage{natbib}     % re-implementation of the LATEX \cite command, to work with both author–year and numerical citations
\usepackage{authblk}    % footnote style author/affiliation input in the \author command
\usepackage{hyperref}   % hypertext links for cross-references
\usepackage[capitalise]{cleveref}   % format of references determined automatically according to the type of reference

\DeclareMathOperator{\vc}{VCdim}    % VC dimension
\DeclareMathOperator{\kl}{KL}       % KL divergence
\DeclareMathOperator{\tv}{TV}       % total variation distance
\DeclareMathOperator{\sgn}{sgn}     % sign function

\newcommand{\bydef}{:=}
\renewcommand{\l}{\ell}            % labeled sample size
\renewcommand{\u}{u}               % unlabeled sample size function
\newcommand{\X}{\mathcal{X}}     % feature space
\newcommand{\Y}{\mathcal{Y}}     % label space
\newcommand{\R}{\mathbb{R}}      % real numbers
\newcommand{\N}{\mathbb{N}}      % natural numbers
\renewcommand{\H}{\mathcal{H}}     % hypothesis class
\newcommand{\Pn}{\mathcal{P}}    % admissible distributions
\newcommand{\oo}[1]{\text{\sf 1}_{[#1]}} % indicator function

\newcommand{\ee}[2]{\mathbb{E}_{#1}\left[#2 \right]}

\newcommand{\rate}{\asymp}
\newcommand{\alg}{\mathcal{A}}
\newcommand{\set}[1]{\{#1\}}
\newcommand{\mixparam}{\frac{1}{2}}

\newtheorem{example}{Example}
\newtheorem{theorem}{Theorem}
\newtheorem{proposition}{Proposition}
\newtheorem{lemma}{Lemma}

\newtheorem{definition}{Definition}

\title{When can unlabeled data improve the learning rate?}

\author[1]{Christina Göpfert}
\author[2]{Shai Ben-David}
\author[3]{Olivier Bousquet}
\author[3]{Sylvain Gelly}
\author[3]{Ilya Tolstikhin}
\author[4]{Ruth Urner}
\affil[1]{Bielefeld University\thanks{Part of this work was completed while the first author was at Google. In addition, funding by the CITEC cluster of excellence and by the BMBF under grant number 01IS18041A is gratefully acknowledged. }}
\affil[2]{University of Waterloo}
\affil[3]{Google Brain}
\affil[4]{York University}
\date{}

\begin{document}

\maketitle

\begin{abstract}%
In semi-supervised classification, one is given access both to labeled and unlabeled data. As unlabeled data is typically cheaper to acquire than labeled data, this setup becomes advantageous as soon as one can exploit the unlabeled data in order to produce a better classifier than with labeled data alone. However, the conditions under which such an improvement is possible are not fully understood yet. Our analysis focuses on improvements in the {\em minimax} learning rate in terms of the number of labeled examples (with the number of unlabeled examples being allowed to depend on the number of labeled ones).
We argue that for such improvements to be realistic and indisputable, certain specific conditions should be satisfied and previous analyses have failed to meet those conditions. We then demonstrate examples where these conditions can be met, in particular showing rate changes from $1/\sqrt{\l}$ to $e^{-c\l}$ and from $1/\sqrt{\l}$ to $1/\l$. These results improve our understanding of what is and isn't possible in semi-supervised learning.
\end{abstract}

\section{Introduction}
The aim of this paper is to study and clarify different frameworks for analysis of semi-supervised learning (SSL) and formal demonstrations of benefits from unlabeled data.
We outline various such frameworks that have been worked with in previous studies and map results from the literature on the theory of SSL to these frameworks.
We hereby highlight how rather subtle differences in setup can lead to opposing conclusions about the formal benefits of unlabeled data.
As a result, we propose and argue for a clean minimax criterion for identification of learning rate changes between supervised learning (SL) and SSL.
That is, classes of data-generating distributions where the minimax expected excess risk of some SSL algorithm converges at a rate faster than that of \emph{any} SL algorithm as the number of labeled examples grows.
Finally, we demonstrate examples of such a strict rate change.

Demonstrating advantages of SSL over supervised learning, requires upper bounding the excess risk of an SSL learner while also providing a lower bound on the risk of any supervised learner under the same distributional assumptions.
We find that studies which provide both SSL upper and SL lower bounds with the gap in-between follow one of \emph{three common patterns} in their analysis which we argue are all restrictive in one way or another.  

The first common approach to showing such improvements is to grant the SSL algorithm access to the true marginal distribution (\cite{castelli1995exponential,niyogi2013manifold,globerson2017effective}). 
We call this approach \emph{improvements via idealistic SSL}. 
Several papers demonstrate problems which are not learnable by any SL algorithm, while exact knowledge of the marginal leads to finite rates. 
Such demonstrations are typically motivated with an argument that access to sufficient amounts of unlabeled data should yield a close enough estimate of the marginal for these phenomena to also occur in a finite data regime.
However, we show in \cref{lem:unlearnable} that such a transition from an unlearnable problem to an excess error guarantee is impossible if an SSL algorithm only sees finitely many unlabelled points (rather than the marginal).
In this work we thus argue for practical settings where the extra information in SSL  is limited to a \emph{finite sample} of unlabelled points and where the amount of  unlabeled data is determined by a \emph{fixed function} of the size of the labeled dataset.

The second, and more involved way to demonstrate the desired behavior is to allow the class of distributions to depend on the number of labeled and unlabeled examples $(\l, u)$ that the SSL algorithm receives. 
This will be referred to as \emph{improvements via sample size dependent classes}. 
\citet[Theorem 20, Lemma 28]{darnstaedt2015investigation} provides a binary classification problem over a discrete domain which is not learnable for any SL algorithm but can be successfully learned by SSL if the support of the unknown marginal happens to be bounded by some function of $u$. Similarly, \cite{singh2009unlabeled} consider regression problems in $\mathbb{R}^d$ under a cluster assumption where the target function is smooth over the \emph{decision sets} and show\footnote{See 3rd and 5th rows in Table on page 6 of \cite{singh2009unlabeled}}
that SSL may achieve rates faster than SL if the minimal \emph{margin} between those sets is larger than roughly $u^{-1/d}$, but smaller than $\l^{-1/d}$.
While these settings are somewhat more practical, they involve learning against a class of distributions that grows with $\l$ and $u$.
That is, the set of distributions on which SL and SSL are compared is adapted precisely to the tasks which the SSL algorithm can solve with the data available, while the SL algorithm cannot.
We instead argue for comparing SL and SSL on a fixed set of tasks.

Both approaches sketched above used knowledge of certain parameters of the marginal distribution for recovering the target function. 
These \emph{relevant parameters} (manifold/discrete support, decision sets) are either given to the SSL algorithm or, in the second scenario, are estimated from unlabeled data, but then only classes of distributions for which the amount of unlabeled data was sufficient, are considered in the comparison.
A third framework of comparison that has been used in theoretical analysis of SSL fixes the class of distributions by putting strong (but independent of $u$) assumptions  on the class of marginals. 
For instance, constant upper bounds on the size of support (which allows to quickly find all the domain points where the marginal puts a lot of mass) or constant lower bounds on the margin (which allows to quickly discern the decision sets). 
We call this approach \emph{improvements via easy marginal estimation}. 
In such a setup, the overall error naturally decomposes into a part from the marginal estimation and a part from label prediction problem.
If the unsupervised (marginal) estimation turns out to be statistically harder than learning the target labeling function with known parameters of the marginal, SSL can use its larger unlabeled budget to eliminate the loss incurred by the former (harder) part of the problem and beat the SL. \citet[last row in Table on page 6]{singh2009unlabeled} demonstrates such a case.

The three approaches discussed above have one important characteristic in common: SSL is considered on a parameterized set of distributions and it becomes possible for the SSL algorithm to estimate the relevant parameters of the unknown marginal to any given precision. 
However, parameter estimation is generally not a requirement for strong predictive performance.
It is interesting to know whether one can demonstrate improvements of SSL without going through an intermediate step of parameter estimation. 
In this paper we give a positive answer to that question, demonstrating an example where estimation of the relevant parameters is impossible and yet the SSL algorithm with access to $u$ unlabeled points achieves rates faster than SL.

The paper is structured as follows. 
In \cref{sec:definitions} we give a formal definition of ``unlabeled data helps'', contrast it to previous (implicit and explicit) definitions and take a closer look at what kind of improvements have been shown in the literature.
In \cref{sec:assumptions} we explore how different types of assumptions affect the ability for SSL to achieve a change of rates.
\Cref{sec:examples} uses a simple two-point example to illustrate that improvements via idealistic SSL or sample size dependent classes do not necessarily translate to rate improvements, and how even if SSL increases the convergence speed for \emph{every} fixed distribution, \emph{minimax rates} may remain unchanged, driving home the need for rigorous minimax upper and lower bound analyses. A modification of the example yields a problem where SSL improves minimax rates from $\frac{1}{\sqrt{\l}}$ to $e^{-\l}$. The very fast SSL rate indicates that unlabeled data essentially ``solves'' the problem. We show that it is possible to achieve more interesting rate changes, namely from $\frac{1}{\sqrt{\l}}$ to $\frac{1}{\l}$ in \cref{sec:intermediate change}. 
\Cref{sec:discussion} provides a discussion of our results and an outlook on future work.

\section{Supervised and semi-supervised rates}\label{sec:definitions}

\subsection{Notation}
In the remainder of this paper, let $\X$ be an arbitrary feature space, and $\Y = \set{0,1}$ the set of labels. Data are distributed according to an unknown distribution $P$ on $\X \times \Y$ with marginal $P_X$ on $\X$. A \emph{classifier} is a function $f:\X \to \Y$. The \emph{risk} of $f$ on the distribution $P$ is measured via the misclassification probability:
\[
R_P(f) \bydef P(f(X) \neq Y)\,.
\]
We call the function $\eta(x) = P(Y=1|X=x)$ the {\em labeling function} associated to $P$ and the corresponding classifier $f_P^*(x)=\oo{\eta(x)\ge \frac{1}{2}}$, the {\em Bayes classifier}. This classifier achieves the minimal possible risk among all classifiers. Given a set $\H$ of classifiers of $\X$, we denote by $R_{P,\H}$ the minimal risk over all $h \in \H$:
\[
R_{P,\H}\bydef \inf_{h\in\H} R_P(h)\,.
\]
We call such a set $\H$ a \emph{hypothesis class}. Note that we are not assuming that $f_{P}^*$ is a member of $\H$. When it is clear from the context, we will drop the dependence on $P$ and write only $R_\H$. The learning objective is to identify a classifier $h$ (not necessarily in $\H$) such that $R_P(h)$ is close to $R_{P,\H}$. In the semi-supervised learning (SSL) setup, we are given a training sample $S$ made up of a set of $\l$ pairs $(X_1, Y_1), \dots, (X_\l, Y_\l)$ drawn i.i.d.\:according to $P$ and $u$ elements $X_1', \dots, X_u'$ drawn i.i.d.\:according to $P_X$.
The \emph{size} of this sample is $(\l, u)$ and its distribution will be denoted by $P^{(\l,u)}:=P^\l\times P_X^u$.
In the supervised learning (SL) setup, $S$ contains only the $\l$ pairs $(X_1, Y_1), \dots, (X_\l, Y_\l)$ and no unlabeled elements. This corresponds to a semi-supervised sample of size $(\l, 0)$. A \emph{semi-supervised learning algorithm} is a set of maps
\[
 A:(\X\times\Y)^\l\times\X^u\to \Y^\X
\]
for every $(\l, u) \in \N$. A supervised algorithm is such a set of maps for $\l \in \N$ and $u=0$. We will denote by $A(S)$ the classifier returned by the algorithm $A$ upon receiving $S$ as an input.

\subsection{Learning Rates}
The objective of this paper is to analyze whether SSL algorithms are more powerful than SL algorithms in a statistical learning sense. Performance of an algorithm is measured in terms of the excess risk $R(A(S))-R_\H$ which is a random variable (as it depends on the sample). One can either study the expectation\footnote{Notice that we only consider the positive part $(x)_+=\max(x,0)$ as we are only interested in controlling the excess risk, even if there can be situations where the algorithm performs better than the best member of $\H$, e.g. when $\H$ is very limited and the algorithm is able to generate more sophisticated classifiers (the so-called non-proper case).} $\ee{S\sim P^\l}{(R(A(S))-R_\H)_+}$ or the tails of this random variable $P_{S\sim P^\l}\bigl(R(A(S))-R_\H\ge \varepsilon\bigr)$.

While it is customary in the SL case to consider the so-called {\em sample complexity}, that is the function $m(\varepsilon, \delta)$ such that
\[
P_{S}\bigl(R(A(S)) - R_\H > \varepsilon\bigr) \leq \delta \qquad \forall \l \geq m(\varepsilon, \delta) \,,
\]
as a way to measure the performance of algorithms, we find it more convenient in the SSL case to work with errors instead. This means we are looking for a bound $\varepsilon$ on the excess risk as a function of the confidence $\delta$ and the sample size $(\ell,u)$.
To simplify the analysis, we will also only consider the expected excess risk (this leads to weaker tail bounds in terms of $\delta$ but we will leave this as a future research direction) and thus study $\ee{S\sim P^{(\l,u)}}{\left(R(A(S))-R_\H\right)_+}$ as a function of $(\l,u)$.

Our aim is to study the so-called minimax behavior of the expected excess risk, that is the performance of the best possible algorithm under the worst possible distribution, over a possibly restricted set $\Pn$ of \emph{admissible distributions}.

\begin{definition}[Minimax expected risk]\label{def:minimax risk}
The \emph{minimax expected excess risk} of learning a hypothesis class $\H$ on a sample of size $(\l, u)$ over the set of admissible distributions $\Pn$ is the expected excess risk of the best algorithm $A$ under the
worst distribution $P$:
\begin{equation}\label{eq:minimax_risk}
L(\l, u, \H, \Pn) \bydef \inf_A \sup_{P \in \Pn} \ee{S\sim P^{(\l,u)}}{\left(R(A(S))-R_{P,\H}\right)_+} \,.
\end{equation}
\end{definition}

\begin{definition}[SL and SSL learnability]\label{def:learnability}
We say that a problem $(\H, \Pn)$ is \emph{SL learnable} if $L(\l, 0, \H, \Pn)$ converges to zero as $\l$ goes to infinity. We say that a problem $(\H, \Pn)$ is \emph{SSL learnable} if, for some function $u: \N \to \N$,  $L(\l, u(\l), \H, \Pn)$ converges to zero as $\l$ goes to infinity.
\end{definition}

As stated, we use the \emph{convergence rate} of the minimax risk as a measure for the hardness of a learning problem. 
For functions $f,g:\N \to \R^{\geq 0}$, we say $f$ has \emph{rate} $g$ and write 
$f(\l) \rate g(\l)$ if there exist $c_1, c_2 \in \R^{>0}$ such that for all $\l$, $c_1 g(\l) \leq f(\l) \leq c_2 g(\l)$. The rate of convergence of $L(\l,0,\H,\Pn)$ as a function of $\l$ will be called the \emph{supervised rate} (SL rate).
In order to compare it to the semi-supervised situation we introduce a function $\u:\N \to \N$ that relates the amount of unlabeled points to the amount of labeled points. The \emph{semi-supervised rate} (SSL rate) for $u$ is the rate of convergence of $L(\l,\u(\l),\H,\Pn)$ as a function of $\l$.

\subsection{Concepts of unlabeled data helping}

Since we measure the hardness of a learning problem by the convergence rate of the expected excess risk, we say that unlabeled data helps if the SSL rate is faster than the SL rate.

\begin{definition}[Unlabeled data helps]\label{def:unlabeled helps}
We say that \emph{unlabeled data helps} to learn $\H$ over the set of admissible distributions $\Pn$ if there exists some $\u: \N \to \N$ such that
\begin{equation}
\lim_{\l \to \infty} \frac{ \inf_{A_{SSL}} \sup_{P \in \Pn} \ee{S\sim P^{(\l,\u(\l))}}{\left(R(A_{SSL}(S))-R_\H\right)_+} }{\inf_{A_{SL}} \sup_{P \in \Pn} \ee{S\sim P^\l}{\left(R(A_{SL}(S))-R_\H\right)_+} } = 0\,.
\end{equation}
\end{definition}

Note that this definition is rather restrictive. Indeed, we are looking for cases where the minimax rate (as a function of the number of labeled examples) is strictly improved by making use of a finite number of unlabeled examples. Hence situations where the expected excess error is improved by a constant factor are ruled out. Even if this constant factor is as large as the VC dimension of $\mathcal{H}$ for example (e.g. if the expected excess error goes from $d/\l$ to $1/\l$ where $d$ is the VC dimension), we still wouldn't consider this sufficient with our definition. In the following, we present some other definitions of ``unlabeled data helps'': the first option is allowing the class of distributions on which SL and SSL compete to depend on the size of the sample.
\begin{definition}[Unlabeled data helps non-uniformly]\label{def:non-uniform help}
We say that \emph{unlabeled data helps non-uniformly} to learn $\H$ over the sequence of distributions $(\Pn_\l)_{\l \in \N}$ if there exists some $\u:\N \to \N$ such that
\begin{equation}
    \lim_{\l \to \infty} \frac{\inf_{A_{SSL}} \sup_{P \in \Pn_{\l}} \ee{S\sim P^{(\l,\u(\l))}}{\left(R(A_{SSL}(S))-R_\H\right)_+} }{\inf_{A_{SL}} \sup_{P \in \Pn_{\l}} \ee{S\sim P^\l}{\left(R(A_{SL}(S))-R_\H\right)_+} } = 0 \,.
\end{equation}
\end{definition}
Alternatively, instead of comparing the performance over each fixed set $\Pn_\l$, we can compare how the performance of SSL progresses compared to the worst-case performance of the SL, resulting in a weaker definition of non-uniform help;
\begin{definition}[Unlabeled data helps weakly non-uniformly]\label{def:weak non-uniform help}
We say that \emph{unlabeled data helps weakly non-uniformly} to learn $\H$ over the sequence of distributions $(\Pn_\l)_{\l \in \N}$ with $\Pn_\l \subseteq \Pn_{\l+1}$ if there exists some $\u:\N \to \N$ such that
\begin{equation}
    \lim_{\l \to \infty} \frac{\inf_{A_{SSL}} \sup_{P \in \Pn_{\l}} \ee{S\sim P^{(\l,\u(\l))}}{\left(R(A_{SSL}(S))-R_\H\right)_+} }{\inf_{A_{SL}} \sup_{P \in \bigcup_{i \in \N} \Pn_{i}} \ee{S\sim P^\l}{\left(R(A_{SL}(S))-R_\H\right)_+} } = 0 \,. 
\end{equation}
\end{definition}
Finally, if we do not wish to concern ourselves with estimating the marginal from unlabeled data, we can assume that SSL receives the marginal distribution as an additional input.
\begin{definition}[Knowing the marginal helps]\label{def:marginal helps}
We say that \emph{knowing the marginal helps} to learn $\H$ over the set of admissible distributions $\Pn$ if
\begin{equation}
    \lim_{\l \to \infty} \frac{\inf_{A_{SSL}} \sup_{P\in \Pn} \ee{S\sim P^\l}{\left(R(A_{SSL}(S, P_X))-R_\H\right)_+} }{\inf_{A_{SL}} \sup_{P\in \Pn} \ee{S\sim P^\l}{\left(R(A_{SL}(S))-R_\H\right)_+} } = 0  \,.
\end{equation}
\end{definition}

\subsection{Related work}

Unlabeled data can be useful in many ways. For example, it has been shown to be helpful in active learning, proper learning and co-training, see e.g. \cite{urner2011access,urner2013probabilistic}. In this contribution, we are interested in passive semi-supervised learning without properness constraints. A significant body of work investigates how unlabeled data can provide constant-factor improvements, how it can be utilized besides improving minimax rates, and under which models and assumptions it cannot help at all. See \cite{kaariainen2005generalization,ben2008does,balcan2010discriminative,lu2009fundamental,seeger2000inputdependent}.

There are a number of papers, for example \cite{ratsaby1995learning,rigollet2007generalization} that formalize connections between marginal and labeling that could lead to unlabeled data being helpful in the learning process and propose SSL algorithms that exploit this connection. While they typically provide upper bounds on the performance of the proposed algorithms, they do not complement these results with lower bounds for SL algorithms and minimax analyses over worst-case distributions. We will see why this type of analysis is important in \cref{non working two point example}. The papers that do show performance gaps under the proposed assumptions usually follow one of the following three \emph{restrictive patterns} for SSL.

Under the first restrictive pattern, which we call \emph{improvement via idealistic SSL}, improvement is not defined using finite samples. Instead, an SSL algorithm is assumed to have access to the marginal data distribution, and these papers investigate whether knowing the marginal helps, as defined in \cref{def:marginal helps}. \cite{niyogi2013manifold,globerson2017effective,castelli1995exponential} show that under this type of SSL, it is possible to improve from unlearnable to rate $\frac{1}{\l}$ or $\frac{1}{\sqrt{\l}}$, depending on whether the labeling is assumed to be realizable. While these are instructive explorations of how unlabeled data \emph{might} help, we prove in \cref{lem:unlearnable} that this kind of result cannot carry over into the realistic setting of limited unlabeled data.

Under the second restrictive pattern, improvements stem from letting the learners fight against different sets of distributions depending on how many labeled and unlabeled samples are available. We call this pattern \emph{improvement via sample size dependent classes}. \cite{singh2009unlabeled,azizyan2013density}\footnote{3rd and 5th rows in Table on page 6 of \cite{singh2009unlabeled}} provide settings where unlabeled data helps non-uniformly \cref{def:non-uniform help}). Here, unlabeled data can be used to partition the feature space into regions on which the regression function is smooth enough to be approximated by a specialized learner with an improved rate. To make this partitioning possible using the unlabeled data available, but not so easy that SL can perform it adequately well using only the labeled samples, the minimal separation between the smooth regions needs to shrink as a function of the labeled and unlabeled sample size.
Note that if the unlabeled data provides a constant factor rate improvement for each $\Pn_\l$, and the constant factor grows without bound, then unlabeled data helps in this definition, even though for each learning problem the improvement is still only a fixed constant factor. This type of example can be constructed by letting $\H$ be of infinite VC-dimension, but setting $\Pn_\l$ such that the covering number of $\H$ for each fixed marginal is small and a $\varepsilon$-cover of $\H$ can be identified for each $\Pn_\l$ using the $\u(\l)$ unlabeled samples, but not the $\l$ labeled samples. Then an SSL algorithm can first learn a small $\varepsilon$-cover and perform ERM over that small set, while an SL algorithm lacks this information.
This type of construction is used in \cite{darnstadt2013unlabeled,darnstaedt2015investigation}, although they only explicitly prove that unlabeled data helps weakly non-uniformly (\cref{def:weak non-uniform help}).

A feature of this construction is to restrict the set of admissible distributions $\Pn$ such that estimating some property of the marginal is beneficial to solving the learning problem, and learning this property is statistically harder than solving the learning problem under knowledge of the property, or solving it without the knowledge in the first place. This construction does not work when the property in question becomes arbitrarily hard to estimate, e.g. because the separation between decision sets goes to zero. \Cref{def:non-uniform help,def:weak non-uniform help} can be seen as consequences of this problem. It can be countered by bounding the hardness of the marginal estimation problem, in the case of the decision set estimation by bounding the margin width away from zero. This leads to rate improvements in \cite{singh2009unlabeled}\footnote{Last row in Table on page 6} in the  pattern which we call \emph{improvements via easy marginal estimation}. We give an example free from this particular restriction in \cref{working two point example}.

In addition, to the best of our knowledge, no previous results show that is is possible to go from rate $\frac{1}{\sqrt{\l}}$ to $\frac{1}{\l}$ (the rates typically encountered in VC-class learning) using unlabeled data. We will show one way of constructing such problems in \cref{prop:main}.

\section{When Does Unlabeled Data Realistically Help?}\label{sec:assumptions}

\subsection{No free learnability}

If we do not restrict the set of admissible distributions $\Pn$ and require SSL to compete with SL on all possible distributions, an improvement of rates is impossible. This is due to the fact that there exist tight lower bounds for learning in the agnostic and realizable case that do not depend on whether or not the learning algorithm has access to the marginal distribution of the data. These bounds demonstrate that for every hypothesis class, there exist marginals for which the best hypothesis is intrinsically hard to learn, so unlabeled data cannot help for unrestricted $\Pn$.

If we do restrict $\Pn$, what types of improvements are possible? The most interesting type of result would be to demonstrate a ``learnability gap'': to show that there exist problems that are not learnable by any supervised learner, but can be learned using a semi-supervised learner. The existence of problem sequences where unlabeled data helps non-uniformly insinuate that a learnability gap is possible. After all, the examples in \cite{darnstadt2013unlabeled,darnstaedt2015investigation} may not be learnable in a supervised fashion, but each $\Pn_\l$ is learnable at the same rate given enough unlabeled data. What these analyses hide is that the \emph{number} of unlabeled samples needed grows without bound as the set of admissible distributions grows. If the unlabeled sample size is a fixed function of the labeled sample size (and does not depend on the target distribution), the benefit of unlabeled data vanished. Admitting an unlabeled sample of arbitrary size is clearly not a practical setting. Is it possible to modify these examples to find a learnability gap with the unlabeled sample size limited as a function of $\l$? As the following Theorem shows, this is not the case.

\begin{theorem}\label{lem:unlearnable}
If a problem $(\H, \Pn)$  is unlearnable in the SL setting, i.e. $L(\l, 0, \H, \Pn)$ does not converge to $0$, then it is also unlearnable in the SSL setting, i.e. for any $\u: \N \to \N$, $L(\l, \u(\l), \H, \Pn)$ does not converge to $0$.
\end{theorem}

\begin{proof}
To avoid cluttered notation, we here omit $\H$ and $\Pn$ from the error rates.
We prove that if there is some $u$ such that $\lim L(\l, \u(\l))=0$ then $\lim L(\l, 0)=0$. Indeed, $L(\l,0)\ge 0$ by definition and for any $\l$ and $\u$, $L(\l + \u(\l), 0) \leq L(\l, \u(\l))$, since an SL algorithm can simply opt to forget the labels of $\u(\l)$ labeled samples and treat them as unlabeled. Furthermore, $L(\ell,0)$ is non-increasing since an algorithm receiving an $\l+1$ sample can always ignore an example. Thus, $L(\l,0)$ is a non-negative, monotonously decreasing sequence with a sub-sequence that is bounded by $L(\l, \u(\l))$, so if $\lim L(\l, \u(\l))=0$ then $\lim L(\l, 0)=0$.
\end{proof}

On the one hand, \cref{lem:unlearnable} shows that results obtained under knowledge of the marginal or access to infinite amounts of unlabeled data do not necessarily carry over to realistic settings. On the other hand it demonstrates that unlabeled data can only provide realistic improvements on problems that are already learnable in a supervised setting.

\subsection{Why we relate the labeling and the marginal}\label{sec:product assumptions}

What \emph{types} of assumptions on $\Pn$ can lead to rate improvements? The admissible distributions proposed in \cite{darnstadt2013unlabeled,globerson2017effective} are ``product assumptions'' in the sense that they restrict the space of admissible marginals and, in the case of \cite{darnstadt2013unlabeled}, the admissible labelings to those realizable by the hypothesis class, but they do not restrict how marginals and labelings relate to each other: \emph{any} combination of an admissible marginal and labeling is an admissible distribution. Note that, as evidenced by the constructions in these papers, independence of marginal and labeling \emph{does not} immediately imply that unlabeled data or knowledge of the marginal distribution is useless for learning. While it has been shown e.g. by \cite{seeger2000inputdependent} that, if the parameters determining marginal and labeling are independent, unlabeled data is not useful in estimating the parameters that determine the labeling, it may still be helpful in finding a low-risk classifier. For example, in the realizable case, unlabeled data can help determine which of several hypotheses that are compatible with a labeled sample is the most representative, as demonstrated by \cite{kaariainen2005generalization}. Is it possible to find a set of admissible distributions such that unlabeled data helps in the sense of \cref{def:unlabeled helps} without relating marginal and labeling?

For classes of finite VC dimension, under realizability assumptions on the labeling, \cite{darnstadt2013unlabeled} answer this question in the negative up to $\log$ factors. They show that if $\H$ is fixed and finite, then there is a supervised learner whose sample complexity for \emph{every} marginal is at most $\mathcal{O}(\ln|\H|)$ times worse than that of a semi-supervised learner with access to the marginal. If $\H$ is not finite, but of finite VC-dimension, the sample complexity benefit of using a learner specialized to the marginal can be at most $\mathcal{O}(\vc(\H) \cdot \log(1 / \varepsilon))$, though they do not provide examples that yield such a gap. \cite{golovnev2019information} provide lower bounds that show that when $\H$ is the class of projections over $\{0,1\}^d$, for any supervised algorithm there exists a subset of marginals where the algorithm requires $\Omega(\vc(\H))$ as many samples as an algorithm with knowledge of the marginal.

Without any restriction on the possible labeling functions, an improvement of SSL rates is equally out of the question: even a small portion of the input space that can be labeled arbitrarily leads to slow SSL rates, as shown by \cref{thm:rich}.

\begin{definition}\label{def:rich}
We say that a family of probability distributions $\Pn$ is \emph{rich for a class $\H$} 
if there exist hypotheses $h, h' \in \H$ and marginal $P_X$ such that $P_X(\{x: h(x) \neq h'(x), h(x) = 0) \neq P_X(\{x: h(x) \neq h'(x), h(x) = 1\})$, and for every $\alpha \in (0, \frac{1}{2})$, $\Pn$ contains $P_\alpha$ and $P_{-\alpha}$ which consist of $P_X$ paired with labeling functions $\eta_\alpha$ and $\eta_{-\alpha}$, that agree with $h$ where $h = h'$ and take values $\frac{1}{2} + \alpha$ and $\frac{1}{2} - \alpha$ respectively where $h \neq h'$.
\end{definition}

\begin{theorem}\label{thm:rich}
Let $\H$ be a class of finite VC dimension. Then for every set of probability distributions $\Pn$ that is rich for $\H$, knowing the marginal does not help to learn $\H$ over the set of admissible distributions $\Pn$.
\end{theorem}

\begin{proof}
We give a detailed proof of \cref{thm:rich} in \cref{appendix:richness proof}. The main idea is that for a suitably chosen constant $a$, the distributions $P_{a/\sqrt{\l}}$ and $P_{-a/\sqrt{\l}}$ have the same marginal, and deciding between them using $\l$ labeled samples has expected excess risk of order $\frac{1}{\sqrt{\l}}$.
\end{proof}

\section{Relating marginal and labeling -- a helpful and a non-helpful example}\label{sec:examples}

In the following we provide an example where knowing the marginal helps (\cref{def:marginal helps}) and it is possible to construct a sequence of  distributions such that unlabeled data helps non-uniformly and weakly non-uniformly (\cref{def:unlabeled helps,def:non-uniform help}), but unlabeled data does not improve minimax rates.

\begin{example}\label{non working two point example}
Let $\X = \{x_1, x_2\}$ and $\H = \set{0,1}^\X$. Then every marginal distribution $P_X$ on $\X$ can be parameterized by $\beta \in \left(-\frac{1}{2}, \frac{1}{2}\right)$ with $P_X^\beta(x_1) = \frac{1}{2} + \beta$, $P^\beta_X(x_2) = \frac{1}{2} - \beta$. Now, for each $P_X^\beta$, restrict $\Pn$ to contain only those $P$, denoted by $P^{\alpha \beta}$, such that $P(Y=1|x_1) = \frac{1}{2} + \alpha = P(Y=0 | x_2)$ with $\alpha \beta > 0$, i.e. restrict the possible labelings such that the Bayes classifier assigns opposite labels to the two points, and labels with 1 the point that is seen more often.
\end{example}
In \cref{non working two point example}, the Bayes classifier is completely determined by the marginal distribution. As such, this is an example where we can observe improvements via idealistic SSL, i.e. knowing the marginal helps (\cref{def:marginal helps}): an SSL algorithm that knows the marginal has expected excess risk zero, while the SL rate of learning $(\H, \Pn)$ is $\frac{1}{\sqrt{\l}}$.

It is also the case that unlabeled data helps non-uniformly and weakly non-uniformly (\cref{def:non-uniform help,def:weak non-uniform help}). Proofs for the following lower and upper bounds can be found in \cref{appendix:two points bounds}. Let $\u(\l) = \l^2$ and set 
\begin{equation}
    \Pn_\l \bydef \{P^{\alpha\beta} \;|\; |\beta| \geq \frac{1}{\sqrt{\l}}\} \,.
\end{equation}
Then the SSL algorithm $A$ that discards the labeled sample and uses the unlabeled sample to determine which point is more likely to appear is essentially predicting the flip of a coin with bias at least $\frac{1}{\sqrt{\l}}$, and its expected excess risk can be bounded by
\begin{equation}
    \sup_{P = P^{\alpha \beta} \in \Pn_\l} \ee{S\sim P^{(\l,\l^2)}}{\left(R(A(S))-R_{P,\H}\right)_+} \leq \sup_{\alpha,\beta} 2 \alpha e^{-2 \l^2 \beta^2} \leq e^{-2 \l}
\end{equation}
using Hölder's inequality, while 
\begin{equation}
    \sup_{P \in \Pn_\l} \ee{S\sim P^{\l}}{\left(R(A(S))-R_{P,\H}\right)_+} \geq \frac{1}{\sqrt{\l}}\,, 
\end{equation}
so 
\begin{equation}
    \lim_{\l \to \infty} \frac{\inf_A \sup_{P \in \Pn_{\l}} \ee{S\sim P^{(\l,\u(\l))}}{\left(R(A(S))-R_{P,\H}\right)_+} }{\inf_A \sup_{P \in \Pn_{\l}} \ee{S\sim P^\l}{\left(R(A(S))-R_{P,\H}\right)_+} } = 0 \,.
\end{equation}

We may be tempted to think that unlabeled data helps learn $\H$ for $\Pn = \bigcup_\l \Pn_\l$ in accordance with \cref{def:unlabeled helps}. After all, the correct labeling can be estimated using only unlabeled samples. However, for \emph{no} function $\u$ will access to $\u(\l)$ unlabeled samples change \emph{minimax rates}. The \emph{relevant parameter} of the marginal that determines the labeling is $\sgn(\beta)$. While $\beta$ itself can be estimated with high accuracy independently of its value, $\sgn(\beta)$ is difficult to estimate when $\beta$ is close to zero. No matter how many unlabeled samples are available, there exists some value of $\beta$ such that the probability of incorrectly estimating $\sgn(\beta)$ is high. Indeed, a minimax analysis shows that the SL and SSL rate are both lower bounded by $\frac{1}{\sqrt{\l}}$ for any $\u$. This difficulty can be overcome as is suggested by the definition of $\Pn_\l$: by bounding $|\beta|$ away from zero.

\begin{example}\label{slightly cheating two point example}
Let $(\H, \Pn)$ be the problem defined in \cref{non working two point example}. For $c \in \left(0, \frac{1}{4}\right)$, let
\begin{equation}
    \Pn' \bydef \{P^{\alpha\beta} \;|\; |\beta| \geq c \} \,.
\end{equation}
The SSL rate on $(\H, \Pn')$ is arbitrarily fast the faster $\u$ grows, while the SL rate is lower-bounded by $\frac{c}{2} e^{-32 \l c ^2}$.
\end{example}

This is an example of \emph{improvements via easy marginal estimation}. It yields a problem where unlabeled data helps improve from one exponential rate to any faster rate. Both SL and SSL can take advantage of the information provided by marginal estimation, but SSL overpowers SL due to the vast amounts of data it can leverage. 

Can we give an example where unlabeled data helps that is free from the common patterns favoring SSL? The difficulty with using SSL to estimate parameters of the marginal is that, in contrast to what occurs e.g. when estimating the bias of a coin in supervised learning, if the properties are hard to estimate this does not imply that the excess risk of making the wrong choice is small. Indeed, as $\beta$ approaches $0$, the probability of incorrectly guessing $\sgn(\beta)$ approaches $\frac{1}{2}$ for each fixed number of unlabeled samples, but the cost of guessing incorrectly is fixed at $\alpha$. This problem can be circumvented by linking the hardness of the learning problem for the supervised learner to the excess risk of the wrong choice: by setting $\alpha = \beta$. By limiting $\Pn$ in this way, we can achieve an improvement of minimax learning rate, without bounding the hardness of the marginal estimation.

\begin{example}\label{working two point example}
Let $\X = \{x_1, x_2\}$ and $\H = \set{0,1}^\X$. Now, restrict $\Pn$ to those distributions such that $P(Y=1 | x_i) = P_X(x_i)$. That is, the Bayes classifier labels $x_1$ and $x_2$ with opposite labels, the noise of each labeling is equal to the noise in choosing $x_i$, and the point that is less likely to be seen is the one which the Bayes classifier labels with 0. 
\end{example}

Let $A$ be the algorithm that disregards all labels, assigns 1 to the $x_i$ that appears more often in the sample and 0 to the $x_i$ that appears less often. The expected excess loss of this algorithm is bounded by $\frac{1}{\sqrt{\l + \u(\l)}}$, so the SSL minimax rate for $\u(\l) = \l^2$ is at least $f(\l) = \frac{1}{\l}$, the SSL minimax rate for $\u(\l) = \l^4$ is at least $f(\l) = \frac{1}{\l^2}$ and the SSL minimax rate for $\u(\l) = \exp(\l)$ is exponential. The SL rate of $(\H, \Pn)$ is $\frac{1}{\sqrt{\l}}$.\footnote{A similar result could be achieved by requiring that $|\frac{1}{2} - P(Y=1 | X = x_i)| \leq |\frac{1}{2} - P(X = x_i)|$.} Proofs of lower and upper bounds can be found in \cref{appendix:two points bounds}.

\Cref{working two point example} is a simple demonstration of a problem where SL and SSL have different \emph{minimax} rates even though the amount of unlabeled data available to the SSL learner is limited by a fixed function. In order to achieve this improvement, we made the marginal fully determine the labeling, and connected the difficulty of learning the marginal to the excess risk of choosing the wrong hypothesis. This allows us to achieve arbitrarily fast learning rates if the number of unlabeled samples available is very large. This artificial example gives some insight into how unlabeled data can help, or fail to help. However, the restrictions on $\Pn$ are so strong that observing labels becomes completely unnecessary. We are interested in problems where unlabeled data provides a significant, but not arbitrary advantage, and collecting at least some labels is unavoidable. The types of rates most common to supervised PAC learning are $\frac{1}{\sqrt{\l}}$ and $\frac{1}{\l}$. Is it possible to construct a problem that has the slower SL rate and the faster SSL rate, but does not become arbitrarily fast? We will show that this is indeed possible. Our proof uses a general-purpose construction that can be used to create further problems with intermediate rate pairs.

\section{A non-trivial rate change}\label{sec:intermediate change}
In the following, we show that there exist problems of learning over classes of finite VC dimension where SSL leads to ``intermediate rates''. That is, we provide an example of a problem where unlabeled data helps improve rates from $\frac{1}{\sqrt{\l}}$ to $\frac{1}{\l}$, but no $\u: \N \to \N$ can make the rate improve beyond $\frac{1}{\l}$.

\begin{theorem}\label{prop:main}
There exists a problem with SL rate $\frac{1}{\sqrt{\l}}$ and SSL rate $\frac{1}{\l}$ that cannot be improved by increasing the amount of unlabeled data.
\end{theorem}

\begin{proof}
The proof of the Theorem relies on mixing two types of problems: one with SL and SSL rate $\frac{1}{\l}$ (i.e. unlabeled data does not help), and another with SL rate $\frac{1}{\sqrt{\l}}$, but very fast SSL rate (i.e. unlabeled data helps a lot). This mixture is defined as follows: Let $(\H_A, \Pn_A)$ be a problem with domain $\X_A$ and $(\H_B, \Pn_B)$ be a problem with domain $\X_B$. Without loss of generality, let $\X_A \cap \X_B = \emptyset$. Then the \emph{mixture problem} $(\H, \Pn)$ is given by
\begin{align*}
\H &\bydef \set{h : \X_A \cup \X_B \to \set{0,1}: h|_{\X_A} \in \H_A, h|_{\X_B} \in \H_B} \,,\\
\Pn &\bydef \set{\mixparam \cdot P_A + \mixparam \cdot P_B: P_A \in \Pn_A,  P_B \in \Pn_B}\,.
\end{align*}

In \cref{mixture lower bound,mixture upper bound} in \cref{appendix: mixture upper bounds proof}, we prove that if we assume that for every $P \in \Pn$ the Bayes classifier is in $\H$, then the SL and SSL rates of the mixture problem are dictated by the respective slower rates of the component problems. Now let problem $A$ be \cref{working two point example} and let problem $B$ be the problem of learning a hypothesis class with finite VC-dimension under the realizability assumption. According to \cite[Theorem 1]{darnstadt2013unlabeled}, $B$ has SL and SSL rate $\frac{1}{\l}$ for any $\u$. The two-point example from \cref{working two point example} has SL rate $\frac{1}{\sqrt{\l}}$ and SSL rate faster than $\frac{1}{\l}$ for $\u(\l) \geq \l^2$. Then the mixture problem has SL rate $\frac{1}{\sqrt{\l}}$ and SSL rate $\frac{1}{\l}$ for any $\u$ with $\u(\l) \geq \l^2$.
\end{proof}

\subsection{How much unlabeled data is enough?}

Our definition of ``unlabeled data helps'' allows the algorithm to have access to unlabeled samples of arbitrary size, as long as it is controlled by some function of the labeled sample size. Is it possible to estimate \emph{how much} unlabeled data is needed for a rate improvement? It is possible to give a lower bound on the required growth of $\u$, depending on the SL rate.

\begin{proposition}
If the SL rate of $(\H, \Pn)$ is $\l^{-\alpha}$ with $\alpha > 0$, and unlabeled data helps for $\u:\N \to \N$, then $\frac{\u(\l)}{\l} \to \infty$, i.e. $\u$ must grow superlinearly.
\end{proposition}

\begin{proof}
Since $L(\l + \u(\l), 0) < L(\l, \u(\l))$, $\frac{L(\l, \u(\l))}{L(\l, 0)}\to 0$ implies that $\frac{L(\l+ \u(\l), 0)}{L(\l, 0)}\to 0$. Applying these conditions to $L(\l, 0) = \l^{-\alpha}$ yields the result.
\end{proof}

\section{Discussion}\label{sec:discussion}

We have highlighted how subtly different concepts of SSL improvements can lead to opposing results on the value of unlabeled data for classification problems in a VC framework. In particular, we have highlighted three common patterns of SSL analyses and demonstrated their restrictions on simple examples. We have also demonstrated that different types of rate improvements through SSL in a finite VC setting are possible in examples free of those restrictive patterns. Our examples are artificial and meant to clarify conditions under which unlabeled data helps. Whether such conditions can arise naturally, or indeed what natural conditions mean in this context, remains an interesting open question. Moreover, it would be interesting to further investigate the difference between settings where SSL helps because it is assumed that certain margin parameters are easy to estimate, versus settings where this is not assumed. Finally, it would be interesting to further examine under what types of assumptions finite amounts of unlabeled data can affect rates for hypothesis classes of infinite VC dimension, and whether this case differs markedly from non-parametric learning.

\bibliographystyle{apalike}
\bibliography{arxiv_submission}

\appendix

%%%%%%%%%%%%%%%%%%%%%%%%%%%%%%%%%%%%%%%%%%%%%%%%%%%%%%%

\section{Proof of \cref{thm:rich}} \label{appendix:richness proof}

Since $\H$ has finite VC dimension, the SL rate of $(\H, \Pn)$ is upper-bounded by $\frac{1}{\sqrt{\l}}$. Showing that the rate of any algorithm with access to the marginal distribution is lower-bounded by $\frac{1}{\sqrt{\l}}$ proves that both the SL and SSL rates are of order $\frac{1}{\sqrt{\l}}$, since the SSL rate cannot be slower than the SL rate. Let $C \bydef \{x: h(x) \neq h'(x)\}$, let $c \bydef  P_X(C)$ and $c' \bydef P_X(\{x: x \in C \wedge h(x) = 1\})$. By requirement, $c' \neq c/2$. Without loss of generality, assume $c' > c/2$. A simple calculation shows that 
\begin{equation}
\kl(P_\alpha^\l, P_{-\alpha}^\l) = 2 c \l \alpha \log \left( \frac{1 + 2 \alpha}{1 - 2 \alpha} \right) \,.
\end{equation}
Using
\begin{equation}
\frac{1 + x}{1 - x} = 1 + \frac{2x}{1-x} \qquad \text{and} \qquad \log(1+x) < x \,\, \forall x > 0\,,
\end{equation}
we find that
\begin{equation}
\kl(P_\alpha^\l, P_{-\alpha}^\l) \leq 2 c \l \alpha \frac{4 \alpha}{1 - 2 \alpha}\,.
\end{equation}
For $\alpha < \frac{1}{4}$, this can be bounded by
\begin{equation}
\kl(P_\alpha^\l, P^\l_{-\alpha}) \leq 16 c \l \alpha^2\,.
\end{equation}
\cite[Theorem 2.2, (i)]{tsybakov2009introduction} shows that for any hypothesis test that decides between $P_\alpha$ and $P_{-\alpha}$, the probability of choosing incorrectly is lower-bounded by $\frac{1 - a}{2}$, where $a \geq \tv(P_\alpha, P_{-\alpha})$ and $\tv$ denotes the total variation distance. Since 
\begin{equation}
\tv(P_\alpha^\l, P_{-\alpha}^\l) \leq \sqrt{\frac{1}{2} \kl(P_\alpha^\l, P_{-\alpha}^\l)}\,,
\end{equation}
the probability of a test choosing incorrectly between $P_\alpha$ and $P_{-\alpha}$ can be lower-bounded by 
\begin{equation}
\frac{1 - \sqrt{8 c \l \alpha^2}}{2}\,. \label{richness prob lower bound}
\end{equation}
Since $P_\alpha$ and $P_{-\alpha}$ have the same marginal distribution, this probability is independent of whether or not the marginal is known to the learner. If $P_\alpha$ is the true underlying distribution, any hypothesis that does not majorly label $C$ with $1$ incurs an excess risk of at least $2 \alpha (c' - c/2)$. Likewise, if $P_{-\alpha}$ is the true underlying distribution, any hypothesis that does not majorly label $C$ with $0$ incurs an excess risk of at least $2 \alpha (c' - c/2)$. Let $\alpha = \frac{1}{\sqrt{32 \l c}}$. 
Then \cref{richness prob lower bound} shows that the probability of choosing incorrectly between $P_\alpha$ and $P_{-\alpha}$ can be lower-bounded by $\frac{1}{4}$, and the expected excess risk of any algorithm can be lower-bounded by
\begin{equation}
\frac{2c' - c}{16\sqrt{2c}} \frac{1}{\sqrt{\l}} \rate \frac{1}{\sqrt{\l}}\,.
\end{equation}

%%%%%%%%%%%%%%%%%%%%%%%%%%%%%%%%%%%%%%%%%%%%%%%%%%%%%%%

\section{Two-point example bounds}\label{appendix:two points bounds}

Let $\X = \{x_1, x_2\}$ and $\H = \set{h_{01}, h_{10}}$ with $h_{01}(x_1) = h_{10}(x_2) = 0$ and $h_{01}(x_2) = h_{10}(x_1) = 1$. Let $\alpha, \beta \in (0, \frac{1}{2})$. Let $P_{\alpha\beta+}$ denote the distribution on $\X \times \{0,1\}$ with
\begin{align*}
    P_{\alpha\beta+}(x_1, 0) = \left( \frac{1}{2} + \beta \right) \left( \frac{1}{2} - \alpha \right) & \qquad P_{\alpha\beta+}(x_2, 0) = \left( \frac{1}{2} - \beta \right) \left( \frac{1}{2} + \alpha \right) \\
    P_{\alpha\beta+}(x_1, 1) = \left( \frac{1}{2} + \beta \right) \left( \frac{1}{2} + \alpha \right) & \qquad P_{\alpha\beta+}(x_2, 1) = \left( \frac{1}{2} - \beta \right) \left( \frac{1}{2} - \alpha \right)
\end{align*}
and let $P_{\alpha\beta-}$ denote the distribution on $\X \times \{0,1\}$ with  
\begin{align*}
    P_{\alpha\beta-}(x_1, 0) = \left( \frac{1}{2} - \beta \right) \left( \frac{1}{2} + \alpha \right) & \qquad P_{\alpha\beta-}(x_2, 0) = \left( \frac{1}{2} + \beta \right) \left( \frac{1}{2} - \alpha \right) \\
    P_{\alpha\beta-}(x_1, 1) = \left( \frac{1}{2} - \beta \right) \left( \frac{1}{2} - \alpha \right) & \qquad P_{\alpha\beta-}(x_2, 1) = \left( \frac{1}{2} + \beta \right) \left( \frac{1}{2} + \alpha \right).
\end{align*}
We can think of $x_1$ and $x_2$ as two coins. Then $P_{\alpha\beta+}$ and $P_{\alpha\beta-}$ are both distributions where the coins are biased in opposite directions with bias size $\alpha$. For $P_{\alpha\beta+}$, the first coin is the one we flip more often, and also the one more likely to show heads, while for $P_{\alpha\beta-}$ this is true for the second coin. In the following, we will show that 
\begin{itemize}
    \item For the set of admissible distributions
\begin{equation}
\Pn_0 = \{ P_{\alpha \beta +}, P_{\alpha \beta -} \;|\; \alpha,  \beta \in \left(0, \frac{1}{2}\right)\}\,,
\end{equation}
the SL and SSL rate are both $\asymp \frac{1}{\sqrt{\l}}$ (for any function $\u$ relating labeled and unlabeled sample size)
    \item For the set of admissible distributions
\begin{equation}
\Pn_1 = \{ P_{\alpha \beta +}, P_{\alpha \beta -} \;|\; \alpha,  \beta \in \left(0, \frac{1}{2}\right), \alpha = \beta\}\,,
\end{equation}
the SL rate is $\asymp \frac{1}{\sqrt{\l}}$ while the SSL rate can become arbitrarily fast.
    \item For the set of admissible distributions 
    \begin{equation}
    \Pn_\l = \{ P_{\alpha \beta +}, P_{\alpha \beta -} \;|\; \alpha \in \left(0, \frac{1}{2}\right), \beta \in \left( \frac{1}{\sqrt{\l}}, \frac{1}{2}\right)\}\,,
    \end{equation}
    \begin{align*}
        \inf_{A_{SSL}} \sup_{P \in \Pn_\l} \ee{S\sim P^{(\l,\l^2)}}{\left(R(A_{SSL}(S))-R_\H\right)_+} &\leq e^{-\l} \\
        \inf_{A_{SL}} \sup_{P \in \Pn_\l} \ee{S\sim P^{\l}}{\left(R(A_{SL}(S))-R_\H\right)_+} &\geq \frac{1}{\sqrt{\l}}
    \end{align*}
    \item For the set of admissible distributions
    \begin{equation}
    \Pn_c = \{ P_{\alpha \beta +}, P_{\alpha \beta -} \;|\; \alpha \in \left(0, \frac{1}{2}\right), \beta \in \left( c, \frac{1}{2} \right)\}\,,
    \end{equation}
    with $c \in \left(0, \frac{1}{4}\right)$, the supervised rate is lower-bounded by $\frac{c}{2} e^{-32\l c^2}$.
\end{itemize}

%%%
\subsection{Lower bounds}\label{sec:lower bounds}

Let $P_0 = P_{\alpha \beta +}^\l \times P^u_{\alpha \beta +|\X}$ and $P_1 = P_{\alpha \beta -}^\l \times P^u_{\alpha \beta -|\X}$. Then a simple calculation shows that
\begin{equation}
    \kl(P_0, P_1) = 2 \l \alpha \log \left( \frac{1 + 2 \alpha}{1 - 2 \alpha}\right) + 2 (\l + u) \beta \log \left( \frac{1 + 2 \beta}{1 - 2 \beta} \right)\,.
\end{equation}
Using
\begin{equation}
    \frac{1 + x}{1 - x} = 1 + \frac{2x}{1-x} \qquad \text{and} \qquad \log(1+x) < x \,\, \forall x > 0\,,
\end{equation}
we find that
\begin{equation}
    \kl(P_0, P_1) \leq 2 \l \alpha \frac{4 \alpha}{1 - 2 \alpha} + 2 (\l + u) \beta \frac{4 \beta}{1 - 2 \beta}
\end{equation}
Now assume that $\alpha, \beta < \frac{1}{4}$, then 
\begin{equation}
    \kl(P_0, P_1) \leq 16 \l \alpha^2 + 16 (\l + u) \beta^2 \,. \label{eq:kl bound}
\end{equation}

We will now lower-bound
\begin{equation}
    \inf_A \sup_{P \in \{P_0, P_1\}} \ee{S\sim P^\l\times P_X^u}{\left(R(A(S))-R_{P,\H}\right)_+}\,.
\end{equation}
Since the hypothesis class contains only two hypotheses, one of them optimal, the expected excess loss for each distribution is equal to the excess loss of the wrong hypothesis times the probability that the algorithm chooses the wrong hypothesis. The excess loss of the non-optimal hypothesis is $2 \alpha$. The minimal probability of choosing the wrong hypothesis is equal to the minimal probability of error in hypothesis testing between $P_0$ and $P_1$. \cite[Theorem 2.2, (i)]{tsybakov2009introduction} shows that for any hypothesis test between $P_0$ and $P_1$, the probability of guessing incorrectly is lower-bounded by $\frac{1 - a}{2}$, where $a \geq \tv(P_0, P_1)$ and $\tv$ denotes the total variation distance. Since 
\begin{equation}
    \tv(P_0, P_1) \leq \sqrt{\frac{1}{2} \kl(P_0, P_1)}\,,
\end{equation}
the probability of a test choosing incorrectly between $P_0$ and $P_1$ can be lower-bounded by 
\begin{equation}
\frac{1 - \sqrt{8 \l \alpha^2 + 8 (\l + u) \beta^2 }}{2}\,. \label{pe1 lower bound}
\end{equation}
Let $\alpha = \frac{1}{8 \sqrt{\l}}$ and $\beta = \frac{1}{8 \sqrt{\l + u}}$. Then \cref{pe1 lower bound} shows that the probability of choosing incorrectly between $P_0$ and $P_1$ can be lower-bounded by $\frac{1}{4}$, so the minimax expected excess risk over the set of admissible distributions $\{P_0, P_1\}$ can be bounded by 
\begin{equation}
    \inf_A \sup_{P \in \{P_0, P_1\}} \ee{S\sim P^\l\times P_X^u}{\left(R(A(S))-R_{P,\H}\right)_+} \geq 2 \frac{1}{8 \sqrt{\l}} \frac{1}{4} = \frac{1}{16 \sqrt{\l}}\,. \label{two dist lower bounds}
\end{equation}
If $P_0, P_1 \in \Pn$,
\begin{equation}
    \inf_A \sup_{P \in \Pn} \ee{S\sim P^\l\times P_X^u}{\left(R(A(S))-R_{P,\H}\right)_+} \geq \inf_A \sup_{P \in \{P_0, P_1\}} \ee{S\sim P^\l\times P_X^u}{\left(R(A(S))-R_{P,\H}\right)_+}\,,
\end{equation}
so \cref{two dist lower bounds} also provides lower bounds for any set of admissible distributions that contains $P_0$ and $P_1$. In particular, 
\begin{itemize}
    \item $P_0, P_1 \in \Pn_0$, so \cref{two dist lower bounds} shows that the semi-supervised rate (and thus the supervised rate) on $\Pn_0$ are lower-bounded by $\frac{1}{16\sqrt{\l}}$,
    \item if $u = 0$, $\frac{1}{8 \sqrt{\l}} = \frac{1}{8 \sqrt{\l + u}}$, so that $P_0, P_1 \in \Pn_1$ and \cref{two dist lower bounds} shows that the supervised rate on $\Pn_1$ is lower-bounded by $\frac{1}{16\sqrt{\l}}$.
    \item $P_0, P_1 \in \Pn_\l$, so  
    \begin{equation}
        \inf_{A_{SL}} \sup_{P \in \Pn_\l} \ee{S\sim P^{\l}}{\left(R(A_{SL}(S))-R_\H\right)_+} \geq \frac{1}{16\sqrt{\l}}
    \end{equation}
\end{itemize}
Applying \cite{tsybakov2009introduction}[Theorem 2.2 (iii)] with the bound on the KL divergence in \cref{eq:kl bound} with $\alpha = \beta = c$ yields the lower bound on the supervised rate for $\Pn_c$.

%%%
\subsection{Upper bounds}\label{sec:upper bounds}
Since $\H$ is finite, ERM learns $\H$ in a supervised fashion over the set of all possible distributions at rate $\frac{1}{\sqrt{\l}}$. The supervised and semi-supervised rates for $\Pn_0$ and $\Pn_1$ cannot be slower than $\frac{1}{\sqrt{\l}}$, yielding upper bounds for $\Pn_0$ and $\Pn_1$.

On the other hand, for each fixed distribution $P_{\alpha \beta}$, an algorithm may try to discard all labels and attempt to solve the learning problem by guessing $\sgn(\beta)$ based on which point appears more often in the sample. The probability that this algorithm chooses the incorrect hypothesis is equal to the probability of a binomial variable with parameters $n = \l + u$ and $p = \frac{1}{2} + \beta $ taking a value of more than $\frac{\l + u}{2}$. This probability can be bounded by $e^{-2\beta^2(\l + u)}$ using Hoeffding's inequality. The excess loss of choosing the wrong hypothesis is $2 \alpha$. If $\alpha = \beta$, as is the case for all distributions in $\Pn_1$, the expected excess risk is bounded by $2 \alpha e^{-2 \alpha^2 (\l+u)}$ and reaches its maximum of $\frac{1}{\sqrt{\l + u}}e^{-\frac{1}{2}}$ for $\alpha = \frac{1}{2\sqrt{\l + u}}$. This can be made smaller than any desired $r(\l)$ by choosing $\u(\l) > e^{-\frac{1}{2}} \frac{1}{r(\l)}$. Applying this algorithm to $\Pn_\l$ with $\u = \l^2$ yields the upper bound $2 \cdot 1 \cdot e^{-2\l}$.

%%%%%%%%%%%%%%%%%%%%%%%%%%%%%%%%%%%%%%%%%%%%%%%%%%%%%%%

\section{Bounds for mixture problems}\label{appendix: mixture upper bounds proof}
Let $(\H_A, \Pn_A)$ be a problem with domain $\X_A$ and $(\H_B, \Pn_B)$ be a problem with domain $\X_B$. Assume that $\H_A$ contains the Bayes classifier for all distributions in $\Pn_A$ and $\H_B$ contains the Bayes classifier for all distributions in $\Pn_B$. Without loss of generality, let $\X_A \cap \X_B = \emptyset$. Let the mixture problem $(\H, \Pn)$ be given by
\begin{align*}
\H &\bydef \set{h : \X_A \cup \X_B \to \set{0,1}: h|_{\X_A} \in \H_A, h|_{\X_B} \in \H_B} \,,\\
\Pn &\bydef \set{\mixparam P_A + \mixparam \cdot P_B: P_A \in \Pn_A,  P_B \in \Pn_B}\,.
\end{align*}

In the following, we prove upper and lower bounds on the SL and SSL rates of the mixture problem. Since a supervised learning rate $r(\l)$ is equivalent to an SSL rate $r(\l)$ with $\u(\l)=0$, we only need to consider SSL rates in our proofs. All results transfer directly to SL rates.

\begin{lemma}\label{mixture lower bound}
If $L(\l, \u(\l), \H_A, \Pn_A)  \geq r_A(\l)$ for $\u$ and $L(\l, \u(\l), \H_B, \Pn_B) \geq r_B(\l)$, then
\begin{equation}
L(\l, \u(\l), \H, \Pn) \geq \max(\frac{1}{2}r_A(\l), \frac{1}{2} r_B(\l))\,.    
\end{equation} 
In particular, the SSL rate of $(\H, \Pn)$ for $\u$ is at least as slow as the slower rate of the mixture components.
\end{lemma}

\begin{proof}
Since the Bayes classifier is contained in the hypothesis class, $R_P(\alg(S))-R_{P, \H} > 0$ for all $P$ and we can omit $(\cdot)_+$ in the minimax expected excess risk. Assume that $L(\l, \u(\l), \H, \Pn)$ is not lower bounded by $\frac{1}{2}r_A(\l)$. Then there exists an algorithm $\alg$ such that for all $P_A \in \Pn_A, P_B \in \Pn_B$, if $P \sim C P_A + (1-C) P_B$, then 
\begin{equation}
    \ee{S\sim P^{(\l, \u(\l))}}{R_P(\alg(S))-R_{P, \H}} < \frac{1}{2} r_A(\l)\,. \label{eq:lower bound proof}
\end{equation} 
Recall that 
\begin{equation}
    R_P(\alg(S))  - R_{P, \H} = \frac{1}{2} (R_{P_A}(\alg(S)|_{\X_A}) - R_{P_A, \H_A}) + \frac{1}{2} (R_{P_B}(\alg(S)|_{\X_B}) - R_{P_B, \H_B})\,,
\end{equation}
so \cref{eq:lower bound proof} implies that
\begin{equation}
    \frac{1}{2} \ee{S\sim P^\l\times P_X^{\u(\l)}}{R_{P_A}(\alg(S)|_{\X_A}) - R_{P_A, \H_A}} < \frac{1}{2} r_A(\l)\,.\label{eq:lower bound proof 2}
\end{equation}
Let $\alg_A$ be the algorithm that learns $(\H_A, \Pn_A)$ from a sample $S'$ by
\begin{enumerate}
    \item discarding part of the sample $S'$, where the number of discarded points follows a binomial distribution,
    \item sampling from $P_B$ as many labeled and unlabeled points as were previously discarded, resulting in a sample $S''$ from $P_A$ and $P_B$, 
    \item applying the algorithm $\alg$ to the new sample $S''$, and
    \item restricting the resulting mapping $\alg(S'')$ to $\X_A$.
\end{enumerate}
Then \cref{eq:lower bound proof 2} implies that $L(\l, \u(\l), \H_A, \alg_A, \Pn_A) < r_A(\l)$, which contradicts the assumption that $r_A(\l)$ is a lower bound on the minimax expected excess loss of learning $(\H_A, \Pn_A)$. The analogous argument for $r_B(\l)$ proves the Lemma.
\end{proof}

\begin{lemma}\label{mixture upper bound}
Let $\u(\l) \geq \l$. If $L(\l, \u(\l), \H_A, \Pn_A) \leq r_A(\l)$ and $L(\l, \u(\l), \H_B, \Pn_B) \leq r_B(\l)$ for $\u$, then $L(\l, \u(\l), \H, \Pn) \leq \max(r_A(\frac{\l}{4}), r_B(\frac{\l}{4})) + \exp(-C \l) + \exp(-C \u(\l))$. In particular, the SSL rate of $(\H, \Pn)$ for $\u$ is at most as slow as the slower rate of the mixture components, if the slower rate is not faster than exponential.
\end{lemma}

\begin{proof}
Since the Bayes classifier is contained in the hypothesis class, $R_P(\alg(S))-R_{P, \H} > 0$ for all $P$ and we can omit $(\cdot)_+$ in the minimax expected excess risk.  Let $\alg_A$ and $\alg_B$ be algorithms corresponding to problems $A$ and $B$ whose rates are bounded by $r_A$ and $r_B$, respectively. Given a sample $S$, let $S_{|\bullet}^l \bydef S \cap \X_\bullet \times \Y$ and $S_{|\bullet}^u \bydef S \cap \X_\bullet$ be the labeled and unlabeled examples for problems $\bullet \in \set{A,B}$, respectively. Further, let $S_{|A} = S_{|A}^l \cup S_{|A}^u$ and $S_{|B} = S_{|B}^l \cup S_{|B}^u$. Then $(\H, \Pn)$ can be solved by the algorithm
\begin{equation}
    \alg(S)(x) \bydef \begin{cases} \alg_A(S_{|A})(x) & x \in \X_A \\ \alg_B(S_{|B})(x) & x \in \X_B\,. \end{cases}
\end{equation}

The excess risk of this algorithm trained on the sample $S$ is
\begin{equation}
    \mixparam \left(R(\alg_A(S_{|A}))-R_{\H_A}\right) + \frac{1}{2} \left(R(\alg_B(S_{|B}))-R_{\H_B}\right)\,.
\end{equation}
The expected excess risk can be bounded by
\begin{equation}
    2e^{-\frac{\l}{8}} + 2e^{-\frac{\u(\l)}{8}} + \mixparam r_A\left(\frac{\l}{4}\right) + \mixparam r_B\left(\frac{\l}{4}\right)\,.
\end{equation}
The exponential terms arise because there is always some (albeit small) possibility that the sample $S$ contains only few examples for one of the problems, in which case the excess loss may be large. The number of labeled and unlabeled samples for problem $A$ in $S$ can be thought of as realizations of binomial random variables $C_\l \sim B(\l, 0.5)$ and $C_u \sim B(\u(\l), 0.5)$. The probability that we see less than $\frac{\l}{4}$ labeled examples for either problem, or less than $\frac{\u(\l)}{4}$ unlabeled samples for either problem, can be bounded by $2e^{-\frac{\l}{8}}$ and $2e^{-\frac{\u(\l)}{8}}$, respectively, using Hoeffding's inequality. The expected excess risk over samples with at least $\frac{\l}{4}$ labeled and $\frac{\u(\l)}{4}$ unlabeled examples for each problem can be bounded using risk bounds for the lowest possible number of samples, and observing that the probabilities of all sample number combinations sum to (less than) one:
\begin{align*}
&\sum_{i=\frac{\l}{4}}^{\frac{3}{4}\l} \sum_{j=\frac{\u(\l)}{4}}^{\frac{3}{4}\u(\l)} P(C_\l = i , C_u=j) \mixparam \left( \ee{S_A\sim P_A^{(i,j)}}{R(\alg_A(S_A))} + \ee{S_B\sim P_B^{(\l-i, \u(\l)-j)}}{R(\alg_B(S_B))}\right) \\
\leq & \sum_{i=\frac{\l}{4}}^{\frac{3}{4}\l} \sum_{j=\frac{\u(\l)}{4}}^{\frac{3}{4}\u(\l)} P(C_\l = i , C_u=j) \mixparam \left( \ee{S_A\sim P_A^{\left(\frac{\l}{4}, \frac{\u(\l)}{4}\right)}}{R(\alg_A(S_A))} + \ee{S_B\sim P_B^{\left(\frac{\l}{4}, \frac{\u(\l)}{4}\right)}}{R(\alg_B(S_B))}\right) \\
\leq & 1 \cdot \mixparam \left( \ee{S_A\sim P_A^{\left(\frac{\l}{4}, \frac{\u(\l)}{4}\right)}}{R(\alg_A(S_A))} + \ee{S_B\sim P_B^{\left(\frac{\l}{4}, \frac{\u(\l)}{4}\right)}}{R(\alg_B(S_B))}\right)\,. \\
\end{align*}
If $\u(\l) \geq \l$, $\frac{\u(\l)}{4} \geq \u\left(\frac{\l}{4}\right)$, so the previous quantity can be bounded by
\begin{align*}
 & \mixparam \left( \ee{S_A\sim P_A^{\frac{\l}{4}} \times P_{X_A}^{\u\left(\frac{\l}{4}\right)}}{R(\alg_A(S_A))} + \ee{S_B\sim P_B^{\frac{\l}{4}}\times P_{X_B}^{\u\left(\frac{\l}{4}\right)}}{R(\alg_B(S_B))}\right) \\
&\leq \mixparam L\left(\frac{\l}{4}, \u\left(\frac{\l}{4}\right),\H_A, \alg_A, \Pn_A \right) + \mixparam L\left(\frac{\l}{4}, \u\left(\frac{\l}{4}\right),\H_B, \alg_B, \Pn_B \right) \\
&\leq \mixparam r_A\left(\frac{\l}{4}\right) + \mixparam r_B\left(\frac{\l}{4}\right)\,.
\end{align*}
\end{proof}
\end{document}